\documentclass [
  letterpaper,
  10 pt,
  conference
] {ieeeconf}

\IEEEoverridecommandlockouts
\usepackage [T1] {fontenc}

\usepackage {
  algorithm,
  algorithmic,
  amsmath,
  amssymb,
  amsfonts,
  balance,
  bm,
  booktabs,
  graphicx,
  url,
  hyperref,
  nicefrac,
  subcaption,
  tikz,
  tikz-3dplot,
  todonotes,
  xcolor,
}

\newtheorem{theorem}{Theorem}

\renewcommand \phi \varphi

\usetikzlibrary{
  calc,
  decorations.pathmorphing
}

\usepackage {cleveref} 
\makeatletter
\newlength{\topfloatpad}
\patchcmd{\@cflt}{\unvbox\@tempboxa}{\vbox{\vskip\topfloatpad\unvbox\@tempboxa}}{}{\@latex@error{Could not patch \string\@cflt}\@ehc}
\patchcmd{\@cflt}{\vskip \textfloatsep}{\vskip \textfloatsep\vskip-\topfloatpad}{}{\@latex@error{Could not patch \string\@cflt}\@ehc}
\patchcmd{\@combinedblfloats}{\unvbox\@tempboxa\vskip-\dblfloatsep}{\vbox{\vskip\topfloatpad\unvbox\@tempboxa}\vskip-\dblfloatsep}{}{\@latex@error{Could not patch \string\@combinedblfloats}\@ehc}
\patchcmd{\@combinedblfloats}{\vskip \dbltextfloatsep}{\vskip \dbltextfloatsep\vskip-\topfloatpad}{}{\@latex@error{Could not patch \string\@combinedblfloats}\@ehc}
\makeatother
\crefname {figure} {Fig.}       {figures}
\Crefname {figure} {Figure}     {Figures}
\crefname {section} {Sect.}     {sections}
\Crefname {section} {Section}   {Sections}
\crefname {equation} {Eq.}      {equations}
\Crefname {equation} {Equation} {Equations}

\newcommand \argmax
  {\operatornamewithlimits {argmax}}

\DeclareMathOperator \Tr {tr}

\title {\LARGE \bf Wave-Robust Passive AUV Localization Using FP-MUSIC}
\author{Usama Saqib$^1$, Ola R\o nning$^1$, and Andrzej Wąsowski$^1$\\ \thanks{ $^{1}$Software Quality Research Group, IT University of Copenhagen, Denmark. {\tt\small \{usamas, oroe, wasowski\}@itu.dk}}}

\date{}

\begin{document}
\maketitle

\begin{abstract}
Localizing an autonomous underwater vehicle without pre-deployed seabed transponders, or direct access to onboard vehicle sensors remains a core challenge.  We present a receiver-passive 3-D localization and spatial mapping system utilizing a single floating surface buoy equipped with a hydrophone array and an inertial measurement unit (IMU). The central difficulty is that surface wave motion induces six-degree-of-freedom (6-DOF) perturbations that rotate the array between snapshots, degrading conventional subspace processing. We resolve this by introducing a fixed-point iterative MUltiple SIgnal Classification algorithm (FP-MUSIC) that uses IMU measurements to de-warp snapshot covariances prior to direction-of-arrival estimation. Furthermore, we employ a subspace-projected wideband matched filter to resolve beacon ranges and use power asymmetry for independent front-back identification. Evaluations across simulated sea states demonstrate that FP-MUSIC substantially reduces localization error relative to uncompensated methods and sustains robust 3-D tracking and vehicle orientation estimation under wave-induced motion. {At moderate sea state, FP-MUSIC increases the 2-m beacon-separation accuracy from approximately $45\%$ to $75\%$.} 
\end{abstract}

\section {Introduction}%
\label{sec:introduction}

\noindent
Autonomous underwater vehicles (AUV) can be tracked with either optical or acoustic positioning systems.  For optical systems, fiducial markers\,\cite{tang2025apriltag} are used to track AUVs. Unfortunately, these markers are reliable only at short range, mostly in clean indoor pools, before the visibility degrades.

Underwater, light attenuates faster than sound\,\cite{bosch2016lightbeacons}, which motivates interest in acoustic positioning.  However, the existing systems, such as long-baseline (LBL) configurations \cite{paull2014}, demand pre-deployed seabed transponders, whose complex installation and calibration in offshore conditions quickly becomes prohibitevely expensive.  Ultra short-baseline (USBL) configuration is another method that utilizes a transceiver-responder setup \cite{auvreview2021}. The transceiver is placed under a deployment ship or sub-sea station while the responder is mounted on an AUV. The USBL system relies on a two-way communication, which requires accurate clock-synchronization for acquiring trajectory data from the AUV. Moreover, the existing USBL systems cannot track multiple AUVs at the same time \cite{huang2025raspi2usbl}.


To eliminate these barriers, we present a {receiver-passive} system in which a single floating surface buoy with a hydrophone array and an IMU is placed. The array \textit{listens} to multiple acoustic beacons attached to an AUV. This way we do not require any seabed infrastructure and physical cables. {Unlike the current state-of-the-art that pings the receivers on the AUV for telemetry information, our method only uses one way communication, i.e., only receiving acoustic signals from the AUV. We use ``passive'' in this receiver-side sense to distinguish the proposed system from existing LBL/USBL operation.} The central challenge is that surface wave motion continuously rotates the array, corrupting the parameter estimation required for high-resolution direction-of-arrival (DOA) estimation. We resolve this via an iterative fixed-point de-warping scheme that removes the wave-induced perturbation from the observation.

This work makes three principal contributions:
\begin{itemize}
    \item We propose a {receiver-passive} system that reconstruct an AUV's pose from received acoustic beacons operating mounted on the robot.
    \item We introduce an iterative fixed-point algorithm that jointly uses the IMU data to compensate the phase shift of the acoustic signals and estimate the DOA of the acoustic beacons for accurate pose of the AUV.
    \item We further demonstrate analytically and empirically that rigid-body \emph{translation} (heave, surge, sway) does not affect the DOA estimation and does not require DOA estimation under plane-wave assumptions, so only rotational degrees of freedom require compensation.
\end{itemize}

\noindent
Unlike the state-of-the-art, the proposed system would be easy to deploy at sea, require one-way communication, it compensates for ocean wave-motion and does not require on-side calibration. {We have shared our code online for the community and for reproducibility.}\footnote{https://anonymous.4open.science/r/FP\_MUSIC-A3A3/}
\looseness -1

\section{Related Work}
\label{sec:related}

\noindent
{Caiti et al.\ \cite{caiti2005floatingbuoys} use a field of surface-floating acoustic buoys to localize an AUV.  Each buoy measures the one-way acoustic time-of-flight from the vehicle. Their algorithm combines the range measurements into a set of admissible 3-D positions.  Their system recovers the AUV's position, not its orientation, as a single range measurement carries no bearing data. Moreover, the influence of wave perturbation is not taken into account explicitly.}
\looseness -1

To recover AUV's orientation, a passive receiver must simultaneously detect and resolve at least two distinct acoustic beacons mounted on the hull (e.g., front and back), and then assign each a persistent front/back identity. Conventional techniques, e.g., delay-and-sum (DAS) beamforming \cite{saito2024design}, is limited by the Rayleigh resolution bound set by the array aperture \cite{vantrees2002} and fails to resolve closely spaced co-channel sources. Adaptive spatial filtering, such as the Minimum Variance Distortionless Response (MVDR) \cite{capon1969}, improves spatial resolution over DAS by minimizing interference and noise power while maintaining unity gain toward the direction the signal is originating from. However, MVDR's resolution remains bounded by sensor signal-to-noise ratio (SNR) and aperture constraints, and it exhibits severe degradation when covariance matrices are corrupted by array steering vector mismatches. To obtain high-resolution separation of two closely spaced acoustic beacons on a compact vehicle frame, subspace-based methods such as MUltiple SIgnal Classification (MUSIC) \cite{schmidt1986} are attractive. Conventional multi-beacon and USBL systems separate co-located sources by frequency- or time-division multiplexing \cite{auvreview2021}, dedicating a distinct band or time slot to each, which consumes bandwidth and complicates the transmitter. {Our formulation instead keeps both beacons in a shared frequency band. This avoids dedicating separate acoustic bands or time slots solely for source identity.}

In order to ensure that the DOA estimation is accurate when influenced by 6-DOF wave motion, IMU data could be used to account for the wave-perturbation on the array. Prior methods assume deliberate array rotation \cite{ucarotation2021}, small mooring oscillations \cite{mooredvsa2024}, slow towed array shape deformation \cite{owsley2021}, or active microphone ego-motion compensation \cite{movinghumanoid}. Our framework eliminates all of these assumptions simultaneously. The surface buoy operates under fully unconstrained, stochastic 6-DOF wave motion, and the AUV beacons transmit purely passively without responding to buoy interrogations. Rather than compensating a single pre-estimated bearing as is the case in USBL systems, motion correction is applied directly to the  covariance matrix itself in MUSIC prior to direction of arrival resolution.


As stated earlier, as the sea state rises, roll and pitch swing the array, which both smears the covariance matrix and systematically \emph{biases} the estimated bearing, and hence the recovered position, relative to the assumed flat array. This compensation is not found in marine or oceanic localization but similar approach can be found in other domains, e.g. image processing. This is directly analogous to motion blur in optical imaging, where camera ego-motion during exposure smears visual features and IMU-guided iterative deblurring restores sharpness in an alternating-optimization loop \cite{fergus2006deblur, joshi2010imu}.

To the best of our knowledge, no existing system passively localize the position of AUVs using floating arrays that compensates the 6-DOF wave motion.

\begin {figure} [t]
\centering
\tdplotsetmaincoords {70} {110}
\begin {tikzpicture} [
  tdplot_main_coords,
  scale = 1.0,
  >=stealth,
  font = \scriptsize,
  scale = 1.7,
]
    \draw[->] (0,0,0) -- (4, 0,0  ) node [anchor=north east] {$x$};
    \draw[->] (0,0,0) -- (0, 2.5,0  ) node [anchor=north west] {$y$};
    \draw[->] (0,0,0) -- (0, 0,0.6) node [anchor=south]      {$z$};

    \coordinate (Pc)     at (0,    0,   0);
    \coordinate (Ps)     at (2.5,1.8,-1.8);
    \coordinate (PcProj) at (0,    0,-1.8);
    \coordinate (XProj)  at (0,  1.8,-1.8);
    \coordinate (PsProj) at (2.5,1.8,0.0);

    \draw[->, decorate, lightgray, decoration={snake, post length=1mm}]
    ($(Ps)!.06!(Pc)$) -- ($(Ps)!.83!(Pc)$);

    \fill [blue!5, opacity=0.4]
         (-0.5,-0.5,-1.8)
      -- ( 3.5,-0.5,-1.8)
      -- ( 3.5, 2.5,-1.8)
      -- (-0.5, 2.5,-1.8)
      -- cycle;

    \draw[dashed, gray, thin] (Ps) -- (PcProj)
      node [midway, sloped, below] { \( \sqrt{x^2 + y^2} \) };

    \draw[dashed, gray, thin] (PcProj) -- (Pc)
      node [midway, left] { \( h \) };

      \draw[dashed, gray, thin] (PsProj) -- (Ps)
      node [midway, right] { \( h \) };
      \draw[dashed, gray, thin] (PsProj) -- (Pc);

    \draw [thick, red]
      (Ps) -- (Pc)
      node [midway, above, sloped, pos=0.46]
      {{ \( r \! = \! \|\bm{p}_s \!\! - \! \bm{p}_c\| \) }};

    \filldraw [black] (Ps)
      circle (0.8pt)
      node [anchor=north east]
      {$\bm p_s$} ;

      \tdplotdrawarc [blue]
      {(0,0,0)}{1.3}{0}{atan2(1.8,2.5)}{anchor=north east}
      {$\theta$};

    \tdplotsetthetaplanecoords{atan2(1.8,2.5)};


    \pgfmathsetmacro{\elev}{atan2(-5.3,sqrt(2.5*2.5+1.8*1.8))}
    \tdplotdrawarc[tdplot_rotated_coords, blue]
    {(0,0,0)}{-1.26}{\elev}{-90}
    {anchor=west, shift={(-2.8pt,0pt)}}
    {$\phi$};

    \draw [black!30, thin]
         (-3.4*0.22, -2.2*0.22, 0)
      -- ( 3.4*0.22, -2.2*0.22, 0)
      -- ( 3.4*0.22,  2.2*0.22, 0)
      -- (-3.4*0.22,  2.2*0.22, 0)
      -- cycle;

    \foreach \col in {0,...,3}{
        \foreach \row in {0,...,3}{
            \pgfmathsetmacro{\xpos}{(\col-1.5)*0.36}
            \pgfmathsetmacro{\ypos}{(\row-1.5)*0.22}
            \filldraw[lightgray] (\xpos,\ypos,0) circle (0.3pt);
        }
    }

    \filldraw [red] (Pc)
      circle (0.8pt)
      node [anchor=south west, shift={(0,0mm)}]
      {$\bm p_c$} ;

    \node [anchor = south east] at (1.5*0.22, -1.1*0.22, 0)
      { \( \bm p _m \) };

\end{tikzpicture}

\caption {A floating uniform rectangular array (URA) of hydrophones used for AUV localization. \label {fig:geometry}}

\end {figure}
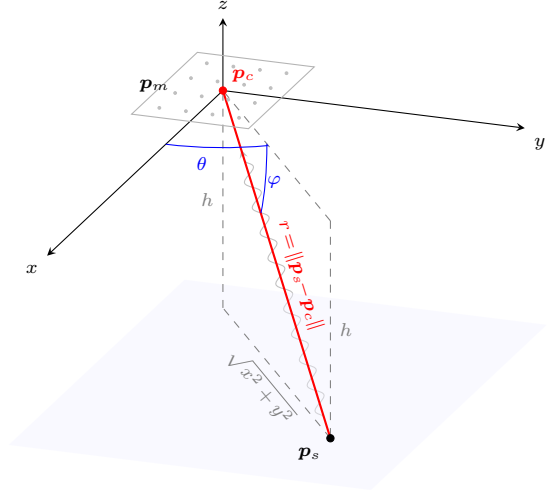

\section{System Model and Problem Formulation}%
\label{sec:sigmodel}

\subsection{Array and Target Geometry}

\noindent
A uniform rectangular array (URA) with $M$ hydrophones floats on the water surface (\cref{fig:geometry}). An underwater source is positioned at \( \bm p _s \in \mathbb R ^3 \) at depth \( h>0 \) in the array coordinate frame, whose  origin \( \bm p _c = \mathbf 0 \) in the array center.  The spatial range $r$, the azimuth $\theta$, and the elevation $\phi$ are:
\begin{align}
  r &= \|\bm p _s - \bm p _c\| = \| \bm p _s \| \label{eq:rng} \\
  \theta &= \operatorname{atan2}(y, x) \label{eq:az} \\
  \phi &= \operatorname{atan2}\!\left(-h,\,\sqrt{x^2+y^2}\right), \label{eq:el}
\end{align}

where $(x, y, -h) = \bm{p}_s - \bm{p}_c = \bm{p}_s$, so the source's vertical coordinate is $z=-h$; since the array floats above the submerged source, and $\phi\in[-90^\circ,0^\circ]$.

\subsection {Transmitted Wideband LFM Chirp}

\noindent
We placed two acoustic beacons on the underwater AUV, one at the front and one at the back, transmitting at different power levels.  Each beacon emits a linear frequency modulated (LFM) chirp.  The windowed transmit pulse is:
\begin{equation}
  s(t) = b(t) \, e^{\left( j 2\pi \left( f_{\text{low}} t + \frac{B}{2 T_{\text{sig}}} t^2 \right) \right)},
\end{equation}
where $T_{\text{sig}}$ is the active signal duration with time \( t \in [0, T_{\text{sig}}] \), \( f_{\text{low}} \) is the lower sweep boundary, $B = f_{\text{high}} - f_{\text{low}}$ is the operational bandwidth, and $j$ is the imaginary unit.  The Hanning window $b(t)$ tapers the pulse edges, keeping the transmitted signal within the intended band. The nominal center frequency of the chirp is defined as \({ f_0 = (f_{\text{low}} + f_{\text{high}}) / 2} \).

\subsection{Received Signal Representation}

\noindent
The waveform received at the m-th hydrophone (m = 1,…,M) from the D co-band acoustic beacons on the AUV (here D = 2, mounted front and back) is:
\begin{align}
  \label{eq:obs-time}
  y_m(t) &= \sum_{i=1,2} (h_{m,i} \ast s_i)(t) + w_m(t) + n_m(t).
\end{align}
Here, $s_i(t)$ is the waveform emitted by beacon $i$,  $h_{m,i}(t)$ is the corresponding acoustic channel to hydrophone $m$, and $w_m$, $n_m$ are sensor and ambient noise, respectively. The beacons transmit in a shared band. They are separated spatially, not by frequency allocation (unlike in USBL).

Under the far-field assumption, the wavefront is planar across the array, so all hydrophones share a single arrival direction $\hat{\bm u}(\theta_i,\phi_i)$. The plane-wave delay at hydrophone $m$ is proportional to its offset $\bm{d}_m = \bm{p}_m-\bm{p}_c$ projected onto that direction. Hydrophones displaced toward the source are reached earlier, so the delay \( \tau_{m,i} \) is:
\begin{align}
  \label{eq:delay}
  \tau _{m,i} (\theta_i,\phi_i)
     &= -\frac{\bm{d}_m\cdot\hat{\bm{u}}(\theta_i,\phi_i)}{c}, \\
  \hat{\bm{u}}(\theta_i,\phi_i)
  &= [
       \cos\phi_i\cos\theta_i;
       \cos\phi_i\sin\theta_i;
       \sin\phi_i
       ]^\top,
  \label{eq:unit_vec}
\end{align}
where $\hat{\bm{u}}(\theta_i,\phi_i)$ is the unit vector from the array center towards beacon $i$ and $c$ is the speed of sound in $\mathrm m/\mathrm s$. The general underwater channel contains direct and reflected arrivals; we model only the direct arrival, deferring multipath to future work.

We work in the frequency domain, where convolution becomes multiplication, so the channel in \cref{eq:obs-time} reduces to a per-hydrophone phase factor; MUSIC likewise operates in this domain. This formulation also leaves the operating band free, so the same method applies to ultrasonic or audible beacons.  Under the direct-path model, the beacon-$i$ transfer function at hydrophone $m$ is:
\begin{equation}
  \label{eq:channel}
  H_{m,i}(f,\theta_i,\phi_i) =
  g_i \, e^{\left(-j2\pi f\tau_{m,i}(\theta_i,\phi_i)\right)},
\end{equation}
where $g_i$ is the gain of beacon $i$, combining its transmit level and propagation attenuation. We apply the Fourier transform to \cref{eq:obs-time}, convolution becomes multiplication, and then substitute \cref{eq:channel} into the result, obtaining:
\begin{multline}
  \label{eq:obs-freq}
  Y_m(f) =
  \sum_{i=1,2}  g_i S_i (f)e^{\left(-j2\pi f\tau_{m,i}(\theta_i,\phi_i)\right)} \\[-2ex]
  +W_m(f)+N_m(f),
\end{multline}
where  $m=1,\dots,M$, $S_i(f)$ is the transmitted wideband signal, and $W_m(f)$ and $N_m(f)$ are the sensor and ambient noise, respectively; all frequency domain counterparts of the terms in \cref{eq:obs-time}.
\looseness -1

Given the received spectra $Y_m(f)$, the task is to estimate the arrival directions $(\theta_i,\phi_i)$ for both beacons, which determine the delays $\tau_{m,i}$ via \cref{eq:delay}. Combined with the wideband ranging stage (\cref{sec:range}), these directions yield the ranges $r_i$, and hence the AUV pose.

\section{DOA via Fixed-Point MUSIC (FP-MUSIC)}
\label{sec:fpmusic}


\noindent\looseness -1
A \emph{frame} is a complete recording window across all $M$ hydrophones (5 seconds in our experiments). The array records $F$ frames in sequence. To estimate the directions-of-arrival, we partition each recorded frame into $L$ short-time snapshots long enough to observe the entire chirp.  For each snapshot, we perform narrowband parameter estimation at frequency $f_0$ (see \cref{sec:sigmodel}). At snapshot $\ell$, we stack the $M$ per-hydrophone spectra $Y_m(f_0)$ from \cref{eq:obs-freq} into the array snapshot vector $\bm{Y}_\ell(f_0) = [Y_1(f_0),\dots,Y_M(f_0)]^\intercal \in \mathbb{C}^{M\times1}$. Similarly, we can rewrite the noise terms \( W_m(f_0) \) and $N_m(f_0)$ as $\bm{W}_\ell, \bm{N}_\ell \in \mathbb{C}^{M\times1}$, and then apply the standard MUSIC\,\cite{schmidt1986}:
\begin{align}
    \label{equ:narrowband_array_model}
    \bm{Y}_\ell(f_0)&=\sum_{i=1,2}\bm{a}_i(f_0)S_{i,\ell}(f_0)+\bm{W}_\ell+\bm{N}_\ell ,
\end{align}
where \(\ell= 1,\dots, L\) and $\bm a_i(f_0)\in \mathbb C^{M \times 1}$ is the
steering vector for beacon $i$ (see below). Since $g_i$ is common to all $M$ hydrophones, it factors out of the steering vector and does not affect DOA estimation; it is therefore absorbed into the source amplitude, $S_{i,\ell}(f_0) \triangleq g_i S_i(f_0)$, beacon $i$'s \emph{received} amplitude at $f_0$ in snapshot $\ell$.

\looseness -1
\Cref{eq:delay} gives the delay at hydrophone $m$ for a wave arriving from direction $(\theta,\phi)$. \Cref{eq:channel} turns it into a phase factor. Collecting the phases across all hydrophones in the array gives the \emph{steering vector} for that direction:
\begin{equation}
  \label{eq:steer}
  [\bm{a}(f, \theta, \phi)]_m
  = e^{\left(-j2\pi f\tau_m(\theta,\phi)\right)}
  = e^{ \left(\,j\frac{2\pi f}{c}\bm{d}_m\cdot\hat{\bm{u}}(\theta, \phi)\right)},
\end{equation}
where $\tau_m(\theta,\phi)$ is the delay of \cref{eq:delay} evaluated at an
arbitrary direction rather than at a beacon direction. The steering vector collects the phases that the $M$ hydrophones would observe for a plane wave arriving from $(\theta,\phi)$. We write $\bm a_i(f) = \bm a(f,\theta_i,\phi_i)$ for the steering vector at beacon $i$'s direction, as used in \cref{equ:narrowband_array_model}.

From the snapshot vectors $\bm{Y}_\ell(f_0)$, standard MUSIC forms the sample spatial covariance and applies power-normalized diagonal loading \cite{hawes2015, benesty2009noise} as:
\begin{align}
  \label{eq:Rcov}
  \bm{R}_{\rm raw} &= \frac{1}{L}\sum_{\ell=1}^{L} \bm{Y}_\ell(f_0)\,\bm{Y}_\ell^H(f_0), \\
  \bm{R} &= \bm{R}_{\rm raw} + \mu\,\tfrac{\Tr(\bm{R}_{\rm raw})}{M}\,\bm{I}_M,
\end{align}
where $\mu$ is a regularization parameter and $\bm{I}_M$ is the identity matrix.  The loading term remedies the case $L < M$, ensuring $\bm{R}$ is full rank and the noise subspace well defined.  Eigendecomposition then yields $\bm{R} = \bm{E}_s\bm{\Lambda}_s\bm{E}_s^H + \bm{E}_n\bm{\Lambda}_n\bm{E}_n^H$,
partitioning the $2$ signal eigenvectors $\bm{E}_s$ (one per beacon) from the noise subspace $\bm{E}_n \in \mathbb{C}^{M\times(M-2)}$. MUSIC locates each
source by searching over $(\theta,\phi)$ for steering vectors $\bm{a}(f_0,\theta,\phi)$ nearly orthogonal to $\bm{E}_n$.

\subsection{IMU Kinematics and Translational Invariance}

\noindent
To account for the perturbation due to wave motion, we formulate the composite 3-D rotation matrix $\bm{Q}_\ell \in SO(3)$ mapping hydrophone positions from the nominal flat rest frame to the perturbed body frame at snapshot $\ell$. This follows the intrinsic ZYX (yaw-pitch-roll) convention:
\begin{equation}
  \bm{Q}_\ell = \bm{Q}_z(\gamma_\ell)\,\bm{Q}_y(\beta_\ell)\,\bm{Q}_x(\alpha_\ell),
\end{equation}
where $\bm{Q}_x(\alpha_\ell)$, $\bm{Q}_y(\beta_\ell)$, and $\bm{Q}_z(\gamma_\ell)$ are the canonical roll, pitch, and yaw rotations.

Accounting for snapshot rotation $\bm{Q}_\ell$, element positions become $\bm{Q}_\ell\bm{d}_m$, giving the motion-compensated steering vector, written $\bm{a}_\ell$ to distinguish it from the rest-frame $\bm{a}$ in \cref{eq:steer}:
\begin{equation}
  \label{eq:compsteer}
  [\bm{a}_\ell(f_0, \theta, \phi)]_m = e^{\!\left( +j\,\frac{2\pi f_0}{c}\,(\bm{Q}_\ell \bm{d}_m)\cdot\hat{\bm{u}}(\theta, \phi) \right)}.
\end{equation}

\begin{theorem}[Translational Invariance]\label{thm:invariance}
Rigid-body array translation $\Delta\bm{p} \in \mathbb{R}^3$ does not alter array covariance matrix under far-field plane-wave assumptions.
\end{theorem}

\medskip

\emph{Proof:} A translation shifts hydrophone positions to $\bm{p}'_{m} = \bm{p}_m + \Delta\bm{p}$ and $\bm{p}'_{c} = \bm{p}_c + \Delta\bm{p}$, leaving rigid body offsets invariant, $\bm{d}'_m = \bm{p}'_m - \bm{p}'_c = \bm{d}_m$. The translation introduces an identical phase factor $e^{-j\varphi_c}$ to every element, where $\varphi_c = \frac{2\pi f_0}{c}\Delta\bm{p}\cdot\hat{\bm{u}}$. In the covariance matrix calculation:
\begin{align}
  \bm{R}'_{\rm raw} &= \frac{1}{L}\sum_{\ell=1}^L \left(e^{-j\varphi_c}\bm{Y}_\ell\right) \left(e^{-j\varphi_c}\bm{Y}_\ell\right)^H \\ &= e^{-j\varphi_c}\left(e^{-j\varphi_c}\right)^{*} \bm{R}_{\rm raw} = \bm{R}_{\rm raw}.
\end{align}
Thus, heave, surge, and sway translations leave covariance-based estimators invariant. \hfill $\square$

\bigskip

\noindent
From the theorem, we can see that only rotational motion requires compensation. This theorem also exhibits that the compensation could be achieved with only a gyroscope instead of an IMU which consists of gyroscope and accelerometer on a single chip.


\subsection{Iterative Fixed-Point Motion Compensation Loop}
\label{sec:em}

\noindent
Inspired by IMU-guided image deblurring, where frame motion is removed iteratively to sharpen distorted visual features, we formulate a fixed-point iteration for $k=1,\dots,K_{\max}$ to unblur spatial snapshot covariance matrices. The proposed method alternates between two blocks.

\bigskip

\noindent
\textbf{Block 1 (Snapshot De-Warping):} Given directional estimates $\{\hat{\bm{u}}_i^{(k-1)}\}_{i=1,2}$ from iteration $k-1$, we de-warp each snapshot for source $i$ by removing the rotational phase perturbation relative to rest/flat geometry, i.e., no perturbation. At snapshot $\ell$ the array is rotated by $\bm{Q}_\ell$, so element $m$ occupies the body offset $\bm{Q}_\ell\bm{d}_m$ from \eqref{eq:compsteer}. Therefore, a plane wave from $\hat{\bm{u}}_i^{(k-1)}$ reaches it with the \emph{measured} phase $\varphi^{\rm rot}_{m,\ell}=\frac{2\pi f_0}{c}(\bm{Q}_\ell\bm{d}_m)\cdot\hat{\bm{u}}_i^{(k-1)}$, while the flat rest array would carry $\varphi^{\rm rest}_{m}=\frac{2\pi f_0}{c}\,\bm{d}_m\cdot\hat{\bm{u}}_i^{(k-1)}$. The perturbation to undo their phase is their difference, where the common $\bm{d}_m$ leaves the residual displacement $(\bm{Q}_\ell-\bm{I}_3)\bm{d}_m$ as:
\begin{equation}
  \label{eq:dewarp_phase}
  \Delta\varphi_{m,\ell} = \varphi^{\rm rot}_{m,\ell}-\varphi^{\rm rest}_{m} = \frac{2\pi f_0}{c}\,(\bm{Q}_\ell-\bm{I}_3)\,\bm{d}_m\cdot\hat{\bm{u}}_i^{(k-1)}.
\end{equation}
Multiplying each snapshot by $e^{-j\Delta\varphi_{m,\ell}}$ then cancels the rotation while retaining the rest steering phase:
\begin{align}
  \label{eq:estep}
  [\tilde{\bm{Y}}_\ell^{(i)}(f_0)]_m = [\bm{Y}_\ell(f_0)]_m 
   e^{\!\left( -j\,\frac{2\pi f_0}{c}\,(\bm{Q}_\ell - \bm{I}_3)\,\bm{d}_m \cdot \hat{\bm{u}}_i^{(k-1)} \right)}
\end{align}
The correction leaves steering phase $\bm{d}_m\cdot\hat{\bm{u}}$ intact, aligning all snapshots to a virtual stationary flat array. The phase alignment observation are used to make the sample covariance matrix \cite{hawes2015, benesty2009noise}:
\begin{equation}
  \label{eq:Rraw_k}
  \hat{\bm{R}}_i^{(k)} = \frac{1}{L}\sum_{\ell=1}^{L}\tilde{\bm{Y}}_\ell^{(i)}(f_0)\tilde{\bm{Y}}_\ell^{(i)H}(f_0), \quad i = 1,2.
\end{equation}
We average these covariances across both sources and apply diagonal loading:
\begin{align}
  \label{eq:Rk}
  \bar{\bm{R}}^{(k)} &= \frac{1}{2}\sum_{i=1,2}\hat{\bm{R}}_i^{(k)}, \\
  \bm{R}^{(k)} &= \bar{\bm{R}}^{(k)} + \mu\,\frac{\Tr(\bar{\bm{R}}^{(k)})}{M}\,\bm{I}_M,
\end{align}

\noindent
Eigendecomposing $\bm{R}^{(k)}$ exactly as in \eqref{eq:Rcov} yields the noise subspace $\bm{E}_n^{(k)}$ and signal-subspace projector $\bm{P}_s^{(k)} = \bm{I}_M - \bm{E}_n^{(k)}(\bm{E}_n^{(k)})^H$. The proposed method resolves the beacons without requiring separate receivers or dedicated frequency bands. Although de-warping for source $i$ leaves a direction-dependent residual phase in other signal components \cite{ziskind1988ap, weiss1990selfcal}, averaging the source-conditioned covariances in \eqref{eq:Rk} prevents favouring single beacon.


\begin{algorithm}[t]
\small
\caption{FP-MUSIC}
\label{alg:fp-music}
\begin{algorithmic}[1]
\REQUIRE For each of $F$ recording frames $n=1,\dots,F$: per-frame snapshot spectra $\{\bm{Y}_\ell(f_0)\}_{\ell=1}^L$, IMU rotations $\{\bm{Q}_\ell\}_{\ell=1}^L$, body offsets $\bm{d}_m$
\STATE $\mathcal{P} \leftarrow \varnothing$
\FOR{$n = 1$ \TO $F$}
  \STATE Initialize $\bm{R}^{(0)}$ from the raw snapshot covariance and $\{\hat{\bm{u}}_i^{(0)}\}_{i=1,2}$ via standard MUSIC \cite{schmidt1986}
  \FOR{$k = 1$ \TO $K_{\max}$}
    \FOR{$i = 1$ \TO $2$}
      \STATE Compute de-warped snapshots $\tilde{\bm{Y}}_\ell^{(i)}$ \eqref{eq:estep} and covariance $\hat{\bm{R}}_i^{(k)}$ \eqref{eq:Rraw_k}
    \ENDFOR
    \STATE Form composite covariance $\bm{R}^{(k)}$ \eqref{eq:Rk}
    \STATE Update directions $\{\hat{\bm{u}}_i^{(k)}\}_{i=1,2}$ via Block 2 \eqref{eq:mstep}
    \IF{\eqref{eq:conv} is satisfied}
      \STATE \textbf{break}
    \ENDIF
  \ENDFOR
  \FOR{$i = 1$ \TO $2$}
    \STATE Compute wideband matched filter envelope $\chi_{\rm env,i}(t)$ \eqref{eq:wb}--\eqref{eq:env} {using the frame timing reference}
    \STATE Estimate range $\hat{r}_i$ \eqref{eq:range} and position $\hat{\bm{p}}_{s,i}$ \eqref{eq:pos}
    \STATE Compute peak level $\mathcal{L}_i$ \eqref{eq:power}
  \ENDFOR
  \STATE Tag $\argmax_i \mathcal{L}_i$ as \texttt{front}, remaining as \texttt{back}
  \STATE $\mathcal{P} \leftarrow \mathcal{P} \cup \{(\hat{\bm{p}}_{s,i}, \text{label}_i)\}_{i=1,2}$
\ENDFOR
\RETURN $\mathcal{P}$
\end{algorithmic}
\end{algorithm}

\bigskip

\noindent
\textbf{Block 2 (MUSIC Update):} Eigendecomposing $\bm{R}^{(k)}$ yields the updated noise subspace $\bm{E}_n^{(k)}$. Updated DOAs are extracted by searching over the flat steering vector using the standard MUSIC pseudo-spectrum \cite{schmidt1986}:
\begin{align}
  \label{eq:mstep}
  (\hat{\theta}_i^{(k)}, \hat{\phi}_i^{(k)})& = \\
  & \argmax_{\theta,\phi} \frac{1}{\bm{a}^H(f_0, \theta, \phi)\,\bm{E}_n^{(k)}(\bm{E}_n^{(k)})^H\,\bm{a}(f_0, \theta, \phi)} \notag.
\end{align}

\noindent
The iteration terminates when the change in the estimated directions falls below a tolerance $\varepsilon$, expressed as an angle between successive unit vectors:
\begin{equation}
  \label{eq:conv}
  \max_{i=1,2}\ \arccos\!\big( \hat{\bm{u}}_i^{(k)} \cdot \hat{\bm{u}}_i^{(k-1)} \big) < \varepsilon.
\end{equation}


\noindent
FP-MUSIC adds at most $K_{\max}$ covariance matrix updates costing $O(LM^2)$ per iteration alongside standard MUSIC operations. While this increases processing load, it achieves full motion compensation without requiring extra hydrophones, additional bandwidth, or a larger array aperture.

\section{Subspace Wideband Ranging and Spatial Mapping}%
\label{sec:range}

\subsection{Subspace-Projected Wideband Filtering}

\noindent
The ranging stage uses the converged signal-subspace projector $\bm{P}_s \triangleq \bm{P}_s^{(K)}$ from \cref{sec:fpmusic} to filter noise from the wideband snapshot vector $\bm{Y}(f) = [Y_1(f),\dots,Y_M(f)]^\intercal \in
\mathbb{C}^{M\times 1}$, the counterpart of $\bm{Y}_\ell(f_0)$ evaluated across
the full band $f \in [f_{\text{low}}, f_{\text{high}}]$ \cite{jiang2021}:
\begin{equation}
  \label{eq:wb}
  Z_i(f) = \frac{1}{M}\sum_{m=1}^{M} [\bm{P}_s\,\bm{Y}(f)]_m\,e^{\!\left(+j\,2\pi f\,{\tau_{m,i}(\hat\theta_i, \hat\phi_i)}\right)}.
\end{equation}
Matched filtering with transmit template $S(f)$ followed by Hilbert envelope extraction produces the range\linebreak[3] decision profile:
\begin{align}
  \label{eq:mf}
  \chi_i(t) &= \mathcal{F}^{-1}\left\{ Z_i(f) \cdot S^*(f) \right\} \\
  \label{eq:env}
  \chi_{\rm env,i}(t) &= \left|\operatorname{\mathcal{H}}\{\chi_i(t)\}\right|,
\end{align}
where $\mathcal{H}$ is the Hilbert operator and $\chi_{\rm env,i}(t)$ is the envelope of the time-domain matched-filter output $\chi_i(t)$. Unlike matched field processing, which correlates the received field against modeled replica fields over a candidate search grid and therefore requires a known environmental model \cite{passivecompact2024}, \eqref{eq:wb}--\eqref{eq:env} operate directly along the already resolved DOA $(\hat\theta_i, \hat\phi_i)$ and require no environmental replica.

\begin{figure*}[t]
\centering
\begin{subfigure}{0.42\textwidth}
    \centering
    \includegraphics[width=\linewidth]{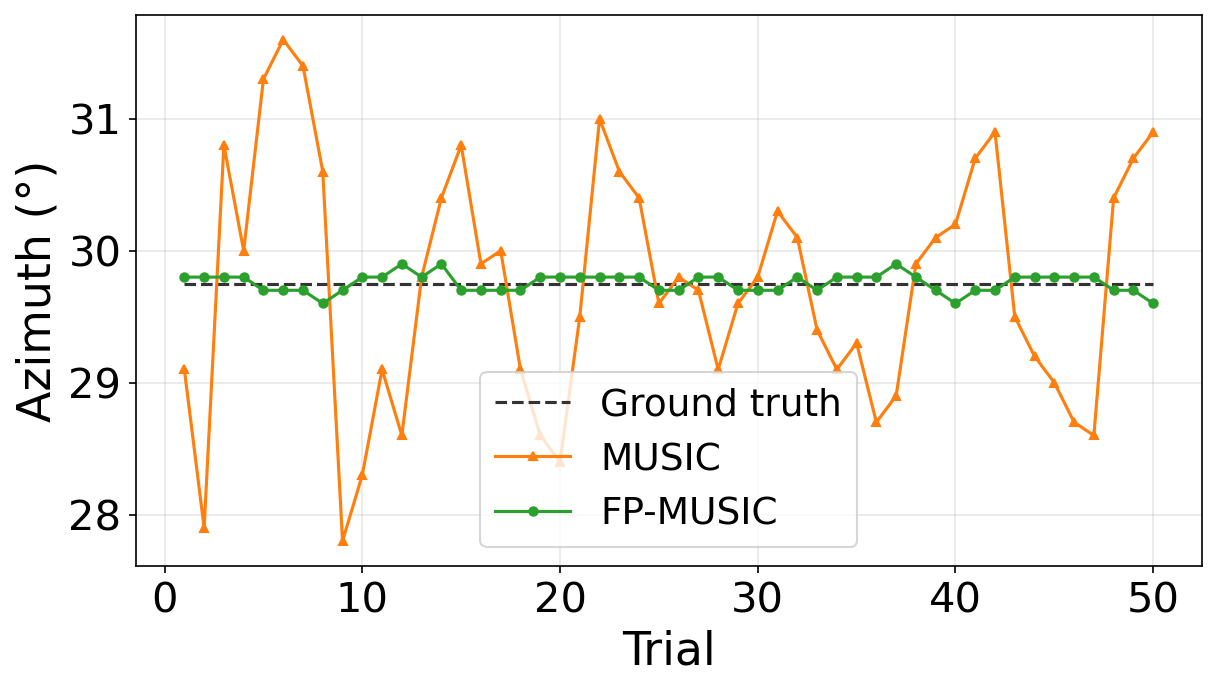}
    \caption{}
    \label{fig:mc_az}
\end{subfigure}
\hspace{1cm}
\begin{subfigure}{0.42\textwidth}
    \centering
    \includegraphics[width=\linewidth]{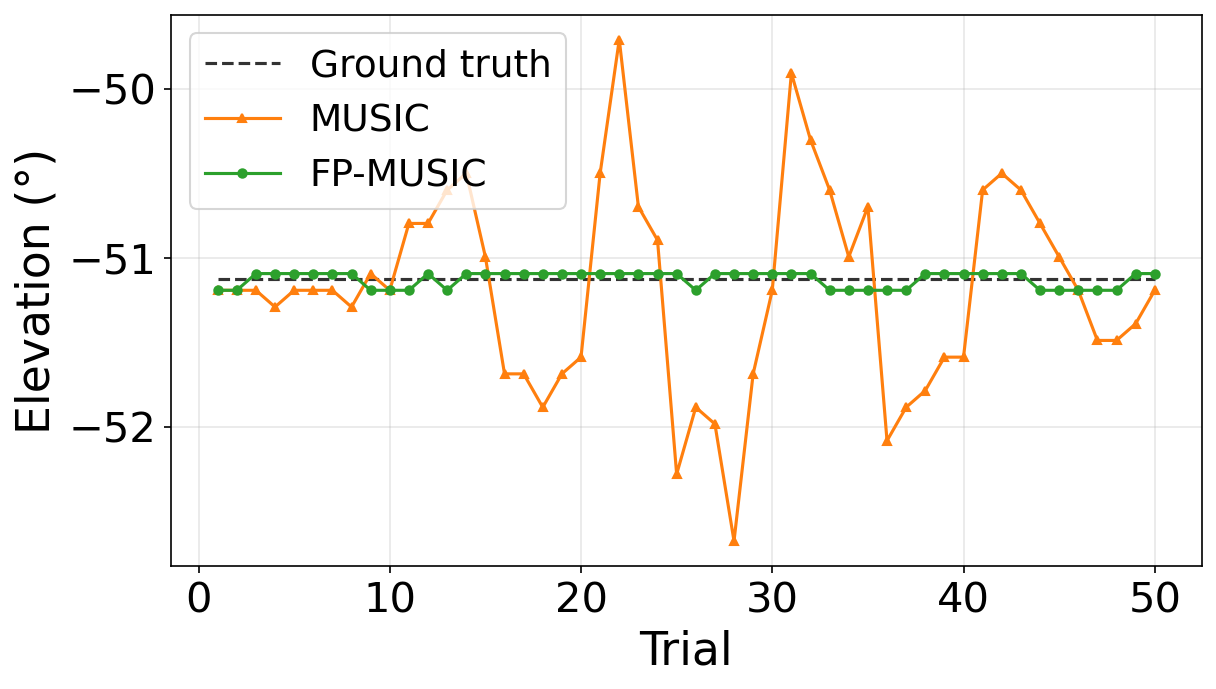}
    \caption{}
    \label{fig:mc_el}
\end{subfigure}
\caption{Performance of MUSIC and FP-MUSIC under 6-DOF wave motion at Sea State 1 using RSME over 50 trials.}
\label{fig:localization_main}
\end{figure*}

\subsection{Range Peak Extraction \& Beacon Tagging}

\noindent
Evaluating $\chi_{\rm env,i}(t)$ at resolved DOA $\hat\theta_i, \hat\phi_i$ yields range and 3-D Cartesian coordinates:
\begin{align}
  \label{eq:range}
  \hat{r}_i &= \argmax_{r}\, \chi_{{\rm env},i}^2(r/c) \\
  \label{eq:pos}
  \hat{\bm{p}}_{s,i} &= \bm{p}_c + \hat{r}_i\, \hat{\bm{u}}(\hat\theta_i, \hat\phi_i).
\end{align}
Here, $r/c$ maps the matched-filter delay to range using the sound known propagation speed, assuming a defined frame time reference to the beacon pulse, estimating an unknown transmitter-receiver clock offset is outside the scope of FP-MUSIC.


Received levels are evaluated directly from peak envelope amplitudes:
\begin{equation}
  \label{eq:power}
  \mathcal{L}_i = 20\log_{10} \chi_{{\rm env},i}(\hat{r}_i/c), \quad i = 1,2.
\end{equation}
The beacon with larger $\mathcal{L}_i$ is labeled \textit{front} and the other \textit{back}. This assigns identities independently per frame. The power offset is used only to resolve the permutation ambiguity after spatial separation, it does not create the MUSIC peaks. Front/back classification accuracy under our power-based tagging scheme is measured empirically in the trajectory-reconstruction experiment.

With the two beacons tagged, the vehicle heading (yaw) and pitch are, respectively, the azimuth and elevation of the front-to-back baseline vector $\hat{\bm{b}} = \hat{\bm{p}}_{s,\mathrm{f}} - \hat{\bm{p}}_{s,\mathrm{b}}$, where $\hat{\bm{p}}_{s,\mathrm{f}}$ and $\hat{\bm{p}}_{s,\mathrm{b}}$ are the tagged front and back positions from \eqref{eq:pos}, obtained from the same
$\operatorname{atan2}$ relations \eqref{eq:az}--\eqref{eq:el} that define $\theta$ and $\phi$, applied to $\hat{\bm{b}}$ rather than to $\bm{p}_s$. Unlike the array-to-source elevation, the baseline elevation is unrestricted in sign, since the vehicle pitches both up and down. Roll is unobservable with two beacons because they form a rotationally symmetric line.

\section{Experiment Results}
\label{sec:experiment}

\noindent
Due to the lack of unified simulators combining surface wave hydrodynamics with multi-channel acoustic signal generation, we evaluated our algorithm using Pyroomacoustics \cite{scheibler2018pyroomacoustics} configured for an underwater acoustic medium ($c = 1500$\,m/s, sampling rate $f_s = 96$\,kHz). Pyroomacoustics was strictly used as a geometric and kinematic snapshot generator to isolate direct-path phase corrections. {Unless otherwise stated, FP-MUSIC, MUSIC, MVDR, and DAS use the same array geometry, acoustic realizations, snapshot count, source locations, and noise realization for each Monte-Carlo trial, so the comparisons isolate the effect of spatial processing and wave-motion compensation.} Although any combination of array configuration could be used but in this work, the receiver array was a $4\times6$ URA ($M=24$) with half-wavelength $5$\,cm element spacing matched to $15$\,kHz. The beacons emitted wideband LFM chirps ($T_{\rm sig} = 0.1$\,s, $f_{\rm low} = 7.5$\,kHz, $f_{\rm high} = 15$\,kHz). Each $5$\,s capture frame was partitioned into $L = 16$ snapshots, which is long enough to observe a full chirp in each snapshot. {The two beacons share the same frequency band and are processed jointly by the array covariance. The $10$ dB transmit power offset is retained for post-localization source labeling. The simulation capture the time reference needed by the matched filter to map each detected pulse to one-way propagation delay.} 6-DOF buoy kinematics were generated using World Meteorological Organization (WMO) Sea State parameters \cite{wmo306, habouche2024}, with significant wave heights $H_s = \{0.05, 0.30, 0.88, 1.88, 3.25, 5.00, 7.50\}$\,m and peak periods $T_p = \{2.0, 3.5, 5.0, 6.5, 8.0, 9.5, 11.0\}$\,s for Sea States 1--7, respectively. The front and back beacons are independent, free-running transmitters that are not phase-locked to one another, so each chirp repetition is emitted with an
independent initial phase and the two sources decorrelate across the $L$ snapshots. The source covariance is therefore full rank, as MUSIC requires, without spatial smoothing, and no aperture is sacrificed to restore rank.

Roll, pitch, and yaw amplitudes were derived from the deep-water maximum wave
slope $\zeta_{\max}$:
\begin{align}
    \zeta_{\max} &= \pi H_s/\Lambda_w, \\
    \Lambda_w &= g_0 T_p^2/2\pi, \\
    \zeta_{\max} &= 2\pi^2 H_s/(g_0\,T_p^2),
\end{align}
where $\Lambda_w$ is the deep-water wavelength, $H_s$ the significant wave height, $T_p$ the peak period, and $g_0 = 9.81$\,m/s$^2$ the gravitational
acceleration \cite{deandalrymple1991}. The IMU data sampled at $400$\,Hz was interpolated to snapshot midpoints. Fixed-point loop parameters were set to $K_{\max} = 20$ and $\varepsilon = 10^{-3}$ degrees. All experiments undergoes  $50$ Monte-Carlos trials.

In practical hardware deployments, modern 9-axis IMUs provide onboard orientation estimation via an Attitude and Heading Reference System (AHRS), outputting time-synchronized unit quaternions $\mathbf{q}(t) = [q_w(t), q_x(t), q_y(t), q_z(t)]^T \in S^3$. At snapshot index $\ell$, the unit quaternion $\mathbf{q}_\ell$ interpolated to snapshot midpoint timestamp $t_\ell$ is mapped directly to the $3 \times 3$ rotation matrix $\bm{Q}_\ell \in SO(3)$ via standard Euler-Rodrigues kinematics. The evaluation of IMU is not the scope of this paper, therefore we assume the AHRS supplies orientation at its nominal factory accuracy.

\begin{figure*}[t]
\centering
\begin{subfigure}{0.32\textwidth}
    \centering
    \includegraphics[width=\linewidth]{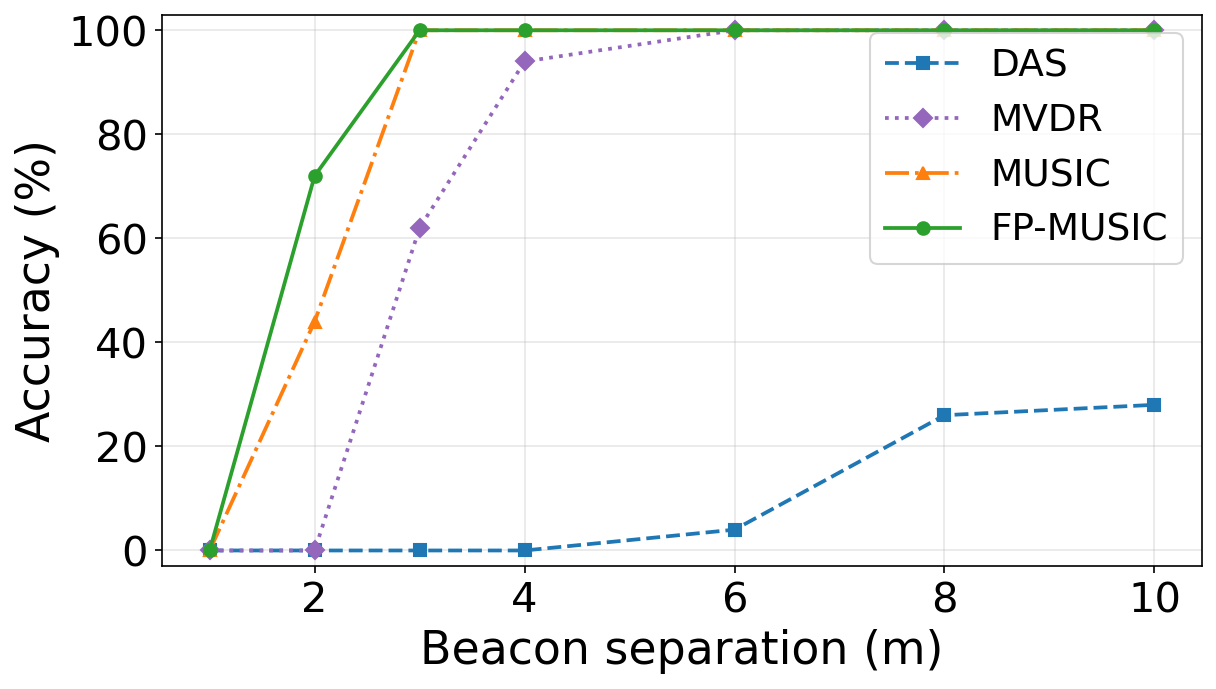}
    \caption{}
    \label{fig:fr_sep}
\end{subfigure}
\hspace{-0.1cm}
\begin{subfigure}{0.32\textwidth}
    \centering
    \includegraphics[width=\linewidth]{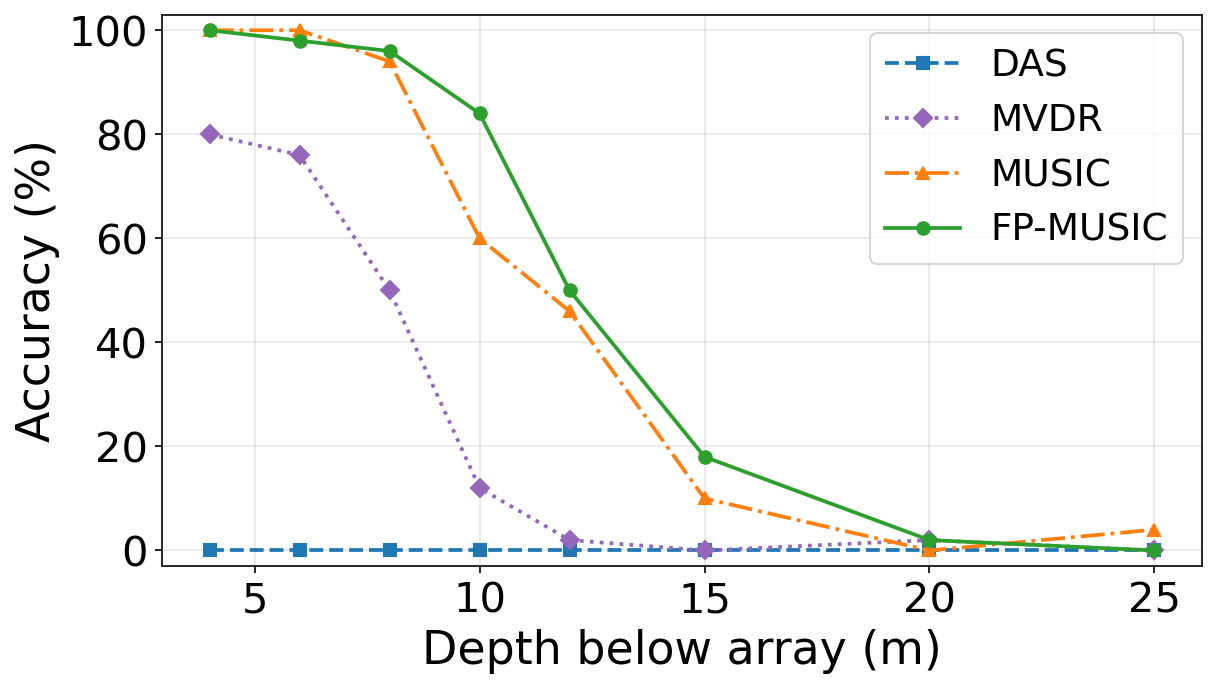}
    \caption{}
    \label{fig:fr_depth}
\end{subfigure}
\hspace{-0.1cm}
\begin{subfigure}{0.32\textwidth}
    \centering
    \includegraphics[width=\linewidth]{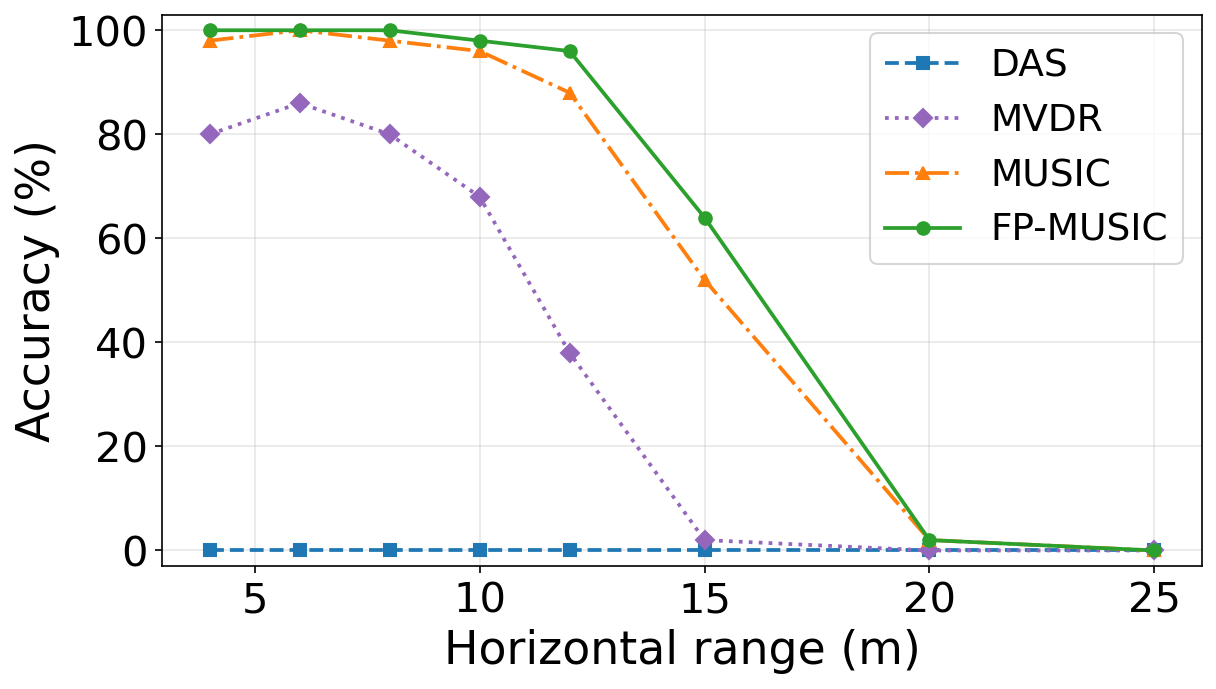}
    \caption{}
    \label{fig:fr_range}
\end{subfigure}
\caption{Evaluation of the method against (a) beacon separation (b) depth and (c) horizontal range.}
\label{fig:performance_em}
\end{figure*}
We structure the evaluation around three research questions. The first is whether IMU-guided snapshot de-warping reduces DOA estimation error under 6-DOF wave motion and up to what sea-state severity. The second question is how two-beacon angular resolvability degrades with AUV depth, horizontal range, beacon separation, and ambient noise against conventional baselines. Finally, whether the combined DOA estimation, wideband ranging, and power-asymmetry tagging pipeline reconstructs an end-to-end AUV's trajectory.


\subsection{Does FP-MUSIC compensate under all sea-state?}

\noindent
To address this, we evaluated FP-MUSIC against uncompensated MUSIC across Sea States 1–7 ($2$ beacons, depth $8$\,m). At Sea State 1 (Fig.~\ref{fig:localization_main}), FP-MUSIC substantially reduced DOA RMSE to $0.25^\circ$ azimuth and $0.02^\circ$ elevation (compared to $0.61^\circ$ and $0.57^\circ$ for standard MUSIC). As summarized in Table~\ref{tab:sea_state_eval}, FP-MUSIC maintains superior accuracy up to Sea State 4. At extreme severity (Sea States 5–7), elevation compensation continues to improve accuracy significantly, whereas azimuth shows minor crossover due to large initialization errors in severe motion. The bold values in Table~\ref{tab:sea_state_eval} indicates best RMSE values (low errors) when compensation is applied to the proposed method compared to standard MUSIC which does not have compensation mechanism.
\looseness -1

\begin{table}[b]
\centering
\caption{Azimuth and elevation DOA RMSE ($^\circ$) by sea state.}
\label{tab:sea_state_eval}
\begin{tabular}{lcccc}
\toprule
& \multicolumn{2}{c}{Azimuth RMSE ($^\circ$)} & \multicolumn{2}{c}{Elevation RMSE ($^\circ$)} \\
\cmidrule(lr){2-3}\cmidrule(lr){4-5}
Sea State & MUSIC & FP-MUSIC & MUSIC & FP-MUSIC \\
\midrule
1 & 0.61 & \textbf{0.25} & 0.57 & \textbf{0.02} \\
2 & 1.21 & \textbf{0.39} & 1.17 & \textbf{0.31} \\
3 & 1.81 & \textbf{1.12} & 1.96 & \textbf{0.87} \\
4 & 4.43 & \textbf{3.32} & 4.20 & \textbf{2.37} \\
5 & 6.95 & \textbf{6.88} & 6.48 & \textbf{4.43} \\
6 & \textbf{10.16} & 10.49 & 8.66 & \textbf{7.35} \\
7 & \textbf{14.75} & 16.09 & 12.34 & \textbf{10.66} \\
\bottomrule
\end{tabular}
\end{table}

\begin{figure*}[t]
\centering
\begin{subfigure}{0.42\textwidth}
    \centering
    \includegraphics[width=1\linewidth]{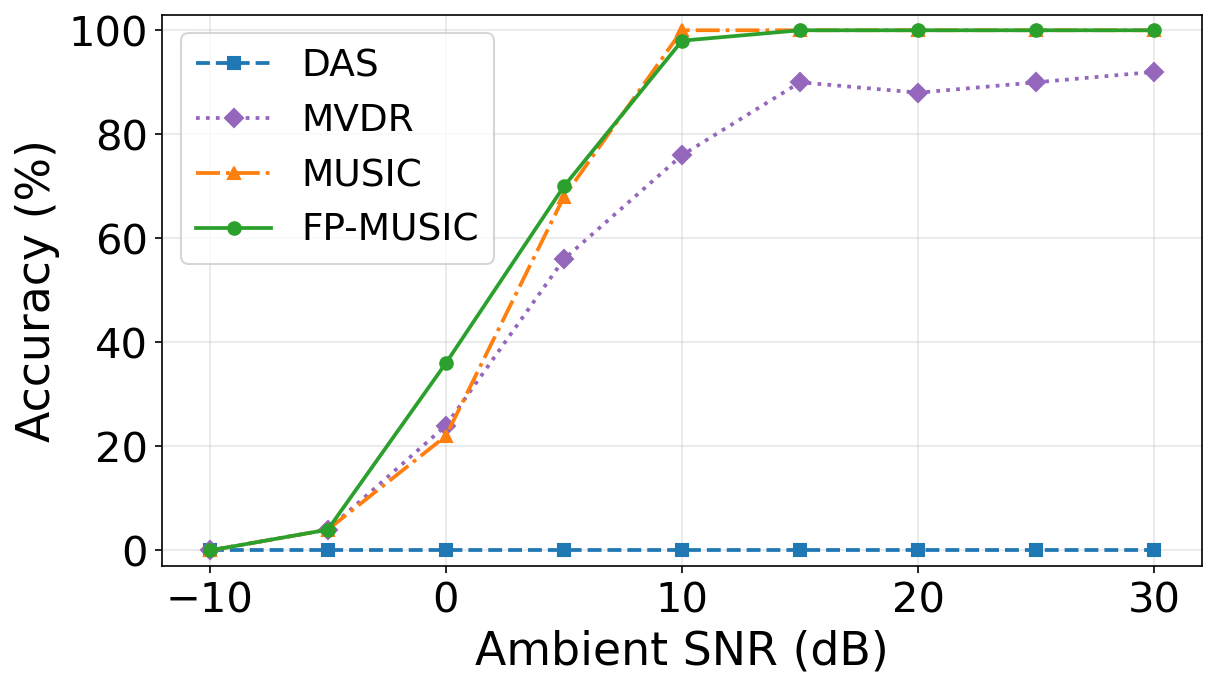}
    \caption{}
    \label{fig:accvsSnr}
\end{subfigure}
\hspace{0.04\textwidth}
\begin{subfigure}{0.40\textwidth}
    \centering
    \includegraphics[width=1.0\linewidth]{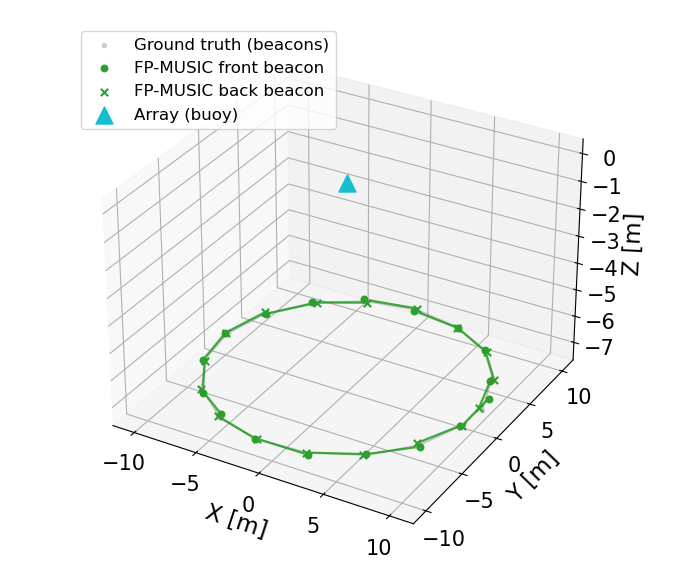}
    \caption{}
    \label{fig:rq3_trajectory}
\end{subfigure}
\caption{(a) Resolvability under oceanic ambient noise and (b) end-to-end trajectory using Algorithm \ref{alg:fp-music}.}
\label{fig:snr_rq3}
\end{figure*}

\subsection{Can the two beacons setup be resolved by FP-MUSIC under varying geometry?}

\noindent
We evaluated the proposed FP-MUSIC algorithm against DAS beamforming, MVDR, and standard uncompensated MUSIC. The aim is to quantify the limits of multi-beacon resolvability as the AUV's physical beacon separation varies, and as the AUV dives deeper or moves further away, which will shrink the relative angular separation at the receiver). In all three experiment, the ambient noise SNR was set to 10\ dB while sensor noise was set to 30\ dB. Resolution accuracy $\eta_{\text{res}}$ is defined as the percentage of successful trials out of total trials:
\begin{equation}
  \label{eq:accuracy}
  \eta_{\text{res}} = \frac{N_{\text{resolved}}}{N_{\text{trials}}} \times 100\%
\end{equation}
where $N_{\text{trials}}$ is the total number of Monte-Carlo trials and $N_{\text{resolved}}$ is the number of trials in which both beacons produce distinct pseudo-spectrum peaks within half their true angular separation. The wave motion was set to Sea State 3.

As shown in Fig.~\ref{fig:performance_em}(a) at the depth of 7\,m and horizontal range of 10\,m, classical DAS fails completely at close beacon separations due to the Rayleigh resolution limit, requiring at least 8\,m of separation to begin resolving the dual beacons. While MVDR narrows the spatial beamwidth, it fails below 4\,m. Subspace methods bypass this limit, however, at a compact separation of 2\,m, uncompensated MUSIC achieves only $\approx 45\%$ accuracy. Under continuous 6-DOF wave motion, intra-frame phase smearing artificially widens the MUSIC pseudo-spectrum peaks, causing tightly spaced beacons to bleed together into a single unresolved peak. By removing this rotational blur, FP-MUSIC restores peak sharpness, boosting 2\,m resolution accuracy to $\approx 75\%$. Increasing the beacons' separation would potentially increase the accuracy range of the AUV.

The Fig.~\ref{fig:performance_em}(b) and (c) illustrate performance as the AUV dives to 25\,m depth and moves to 25\,m horizontal range with at a beacon separation of 3\,m and depth as well as horizontal range of 7\,m and 10\,m, respectively. As range and depth increase, the effective angular separation between the front and back beacons relative to the surface buoy shrinks drastically. Consequently, DAS and MVDR fails almost immediately. While both MUSIC variants eventually succumb to the diminishing angular separation, FP-MUSIC consistently extends the operational envelope, maintaining a 10--15\% accuracy advantage over standard MUSIC in the critical transition regions (10--15\,m) by preserving subspace orthogonality under wave perturbation. Increasing the beacons' separation would increasing the depth as well as horizontal range of the AUV detection.

\subsection{Is Two-Beacon Resolvability Robust to Oceanic Ambient Noise?}

\noindent
To evaluate the robustness of the proposed framework against ambient sea noise, $n_m(t)$, we analyzed the multi-beacon resolution accuracy across a range of ambient SNR from $-10$ to $30$\,dB. For each operating point, the ambient-noise variance is set to $\sigma_n^2 = P_{\rm sig}\,10^{-\mathrm{SNR}_{\mathrm{dB}}/10}$, where $P_{\rm sig}$ is the noise-free received signal power.

As illustrated in the results Fig. \ref{fig:accvsSnr}, DAS beamforming fails to resolve the targets across all evaluated SNR levels, as the physical beacon separation remains fundamentally constrained by the array's Rayleigh limit. The adaptive MVDR beamformer shows high sensitivity to noise, only gradually improving and plateauing at approximately $90\%$ accuracy even in high-SNR. In contrast, the subspace-based methods demonstrate strong resilience to ambient noise. Both standard MUSIC and the proposed FP-MUSIC algorithm exhibit a sharp performance improvement between $-5$ dB and $10$ dB. This experiment demonstrates that iteratively removing wave-induced spatial blur helps preserve signal subspace orthogonality even when ambient noise power is high.
\looseness -1


\subsection{Can FP-MUSIC Reconstruct an End-to-End 3-D Trajectory and Pose?}

\noindent
To answer this, we run Algorithm \ref{alg:fp-music} as the AUV executed a circular underwater trajectory under 6-DOF wave motion, reconstructing both beacon positions and the vehicle orientation per frame from the full proposed pipeline and comparing against ground truth. Fig.~\ref{fig:rq3_trajectory} shows the result qualitatively: the FP-MUSIC pose estimates trace the true 3-D path, and the recovered front/back beacons stay aligned with the ground-truth orientation over the whole session. Using power-asymmetry tagging \eqref{eq:power}, the front and back beacons were classified correctly in $100\%$ of the $50$ resolved frames without relying on inter-frame history. Table~\ref{tab:rq3_pose} reports the per-frame RMSE against ground truth for the same session, recorded at Sea State~2 with the vehicle at $7$\,m depth and $10.2$\,m horizontal range. Compensation reduces the angular error by roughly a factor of around $2$, i.e. azimuth falls from $0.78^\circ$ to $0.25^\circ$ at the front beacon and from $1.00^\circ$ to $0.36^\circ$ at the back, and elevation from $1.02^\circ$ to $0.38^\circ$ and $1.04^\circ$ to $0.37^\circ$ respectively, halving the 3-D position error to $0.12$\,m and $0.14$\,m. Range is unchanged at $0.07$\,m and $0.10$\,m, as expected because the matched filter operates along the already-resolved DOA \eqref{eq:wb}, so residual bearing error well below a degree cannot shift the envelope peak by a range bin.

\begin{table}[t]
  \centering
  \small
  \setlength{\tabcolsep}{5pt}
  \caption{RMSE of AUV trajectory reconstruction under Sea State 2}
  \label{tab:rq3_pose}
  \begin{tabular}{l c c c}
    \toprule
    Quantity & MUSIC & FP-MUSIC \\
    \midrule
    Azimuth, $\hat\theta_{\rm f}$ ($^\circ$) & 0.78 & \textbf{0.25} \\
    Azimuth, $\hat\theta_{\rm b}$ ($^\circ$) & 1.00 & \textbf{0.36} \\
    \addlinespace[2pt]
    Elevation, $\hat\phi_{\rm f}$ ($^\circ$) & 1.02 & \textbf{0.38} \\
    Elevation, $\hat\phi_{\rm b}$ ($^\circ$) & 1.04 & \textbf{0.37} \\
    \addlinespace[2pt]
    Range, $\hat r_{\rm f}$ (m) & \textbf{0.07} & 0.07 \\
    Range, $\hat r_{\rm b}$ (m) & 0.10 & \textbf{0.10} \\
    \addlinespace[2pt]
    Position, $\hat{\bm p}_{s,\rm f}$ (m) & 0.27 & \textbf{0.12} \\
    Position, $\hat{\bm p}_{s,\rm b}$ (m) & 0.30 & \textbf{0.14} \\
    \midrule
    Resolved frames ($N$) & 50 & 50 \\
    \bottomrule
  \end{tabular}
\end{table}

\section{Conclusion and Future Work}%
\label{sec:conclusion}

\noindent
We presented FP-MUSIC, a {receiver-passive, one-way} 3-D localization and spatial mapping framework operating from a wave-driven surface buoy equipped with a hydrophone array and IMU. By proving translational invariance and employing an iterative fixed-point de-warping loop inspired by computational image unblurring, FP-MUSIC compensates for rotational perturbations and {recovers the high-resolution DOA capability of MUSIC under wave-induced covariance distortion without increasing the physical aperture or requiring frequency-separated beacons}. Combined with subspace wideband matched filtering and power-asymmetry tagging, the system achieves robust 3-D tracking and orientation estimation up to Sea State 4. Future work includes validating the system in field trials with a physical buoy array prototype, which will also quantify the effect of real AHRS orientation error, IMU--acoustic timing offset, and multipath propagation which are not captured by the present idealized direct-path channel. Moreover, multi-robot localization will also be investigated in the future iteration of this research.

\section{Acknowledgments}

\noindent
This work was supported by a research grant (VIL77281) from Villum Fonden. A language model was used for grammar check, literature review, and code fix. The proposed method is built by the authors independently.

\balance

\bibliographystyle{IEEEtran}
\bibliography{IEEEabrv,ref}

\end{document}